\documentclass[runningheads]{llncs}
\usepackage[T1]{fontenc}

\usepackage{graphicx}
\graphicspath{ {./images/} }
\usepackage{hyperref}
\usepackage{url}
\usepackage{subcaption}
\usepackage{booktabs, multirow} 
\usepackage{soul}

\usepackage{amsmath}
\usepackage{amsthm}
\usepackage{bm}
\usepackage{dsfont}
\usepackage{amssymb}
\usepackage{algorithm}
\usepackage[noend]{algpseudocode}
\usepackage{wrapfig}

\newcommand{\mb}[1]{\mathbf{#1}}

\newcommand{\md}[1]{\mathds{#1}}

\begin{document}
\title{PCA-Boosted Autoencoders for Nonlinear Dimensionality Reduction in Low Data Regimes}
%
%


\author{Muhammad Al-Digeil\inst{1} 
\and Yuri Grinberg\inst{1} 
\and Daniele Melati \inst{3} 
\and Mohsen Kamandar Dezfouli \inst{2}
\and Jens H. Schmid \inst{2} 
\and Pavel Cheben \inst{2} 
\and Siegfried Janz \inst{2} 
\and Dan-Xia Xu \inst{2}}
\authorrunning{Al-Digeil et al.}
%
\institute{Digital Technologies Research Centre, National Research Council of Canada,
Canada \and
Advanced Electronics and Photonics Research Centre, National Research
Council of Canada, Canada
\and
Université Paris-Saclay, CNRS, France}
\maketitle              
\begin{abstract}
Autoencoders (AE) provide a useful method for nonlinear dimensionality reduction but are ill-suited for low data regimes. Conversely, Principal Component Analysis (PCA) is data-efficient but is limited to linear dimensionality reduction, posing a problem when data exhibits inherent nonlinearity. This presents a challenge in various scientific and engineering domains such as the nanophotonic component design, where data exhibits nonlinear features while being expensive to obtain due to costly real measurements or resource-consuming solutions of partial differential equations.

To address this difficulty, we propose a technique that harnesses the best of both worlds: an autoencoder that leverages PCA to perform well on scarce nonlinear data. Specifically, we outline a numerically robust PCA-based initialization of AE, which, together with the parameterized ReLU activation function, allows the training process to start from an exact PCA solution and improve upon it. A synthetic example is presented first to study the effects of data nonlinearity and size on the performance of the proposed method. We then evaluate our method on several nanophotonic component design problems where obtaining useful data is expensive. To demonstrate universality, we also apply it to tasks in other scientific domains: a benchmark breast cancer dataset and a gene expression dataset. 

We show that our proposed approach is substantially better than both PCA and randomly initialized AE in the majority of low-data regime cases we consider, or at least is comparable to the best of either of the other two methods.

\end{abstract}

\section{Introduction}
Principal Component Analysis (PCA) \cite{pca} is a classical dimensionality reduction approach with provably optimal performance for linearly dependent data under mild conditions \cite{tipping1999probabilistic}. Its efficiency clearly suffers  when data exhibits curvature, as PCA cannot distinguish between nonlinear structure and lack of structure in data. In many domains where dimensionality reduction is desired, the data is inherently nonlinear and therefore nonlinear methods are in principle expected to yield superior results. 

Autoencoders are standard neural networks(NN) with the distinct feature that their inputs and outputs are of the same dimension $n$, while a smaller number of neurons, $p$, is placed in an intermediary layer,($p < n$) \cite{kramer1991nonlinear}. The result is that the layers preceding the bottleneck act as an \textit{encoder} and the layers proceeding it are the \textit{decoder} for the reduced representation in the bottleneck. As such, the size of the bottleneck controls the degree of dimensionality reduction achieved. 

While autoencoders are well-suited for nonlinear dimensionality reduction, their typical data requirements are significantly larger than PCA when trained with randomly initialized weights. Yet, it is well known that the behaviour of PCA with $p$ components can be modeled using an autoencoder with a bottleneck of the same number of neurons and linear activation functions \cite{baldi1989neural}. In this paper, we explore the use of PCA to initialize the weights of an autoencoder in a provably numerically stable way. Specifically, we propose PCA-Robust autoencoders (PCA-Robust) which, given a PCA dimensionality reduction to $p$ dimensions, are initialized with PCA-derived weights with linear activation functions and a bottleneck of size $p$ thus replicating the PCA behaviour. The activation function is then allowed to change smoothly to a nonlinear one as part of the normal training process, enabling the network to capture nonlinear behavior to the extent that the data supports it. 

There have been efforts to use PCA initialization of autoencoders \cite{seuret}. However, previous work differs from our study in three important ways: (i) they add a nonlinear single layer to the NN to approximate PCA (linear) results, while we initialize linear multilayer NN to replicate PCA and then allow the training to smoothly break the linearity of NN by learning the activation functions (ii) their primary interest is to speed up a slow training process rather than our concern for data scarcity (iii) they focus predominantly on images which present the autoencoder with much larger vectors and data sizes than we consider here. 

Various other  nonlinear dimensionality reduction methods exist such as t-SNE \cite{hinton2002stochastic}, Kernel PCA \cite{scholkopf1997kernel}, Locally Linear Embedding \cite{roweis2000nonlinear}, Isomap \cite{tenenbaum2000global},  etc. (see\cite{gorban2008principal}). However, in this study we focus on autoencoders for several reasons. First, they provide a seamless way of reconstructing data from the reduced to the original space, a necessary requirement of its application to nanophotonic design. Indeed, this reconstruction is an integral part of the learning process. Second, autoencoders offer a built-in way to map out-of-sample points on the lower dimension. This paves the way for developing adaptive data acquisition schemes that minimize the number of samples required to achieve a desired level of accuracy. Finally, Neural Network implementations are widely available, making their use easily accessible to communities outside the research domain.

We present the methodology as well as the algorithmic details of the robust PCA-based initialization, accompanied by the correctness and stability proofs. To study the impact of the proposed approach we first analyze its performance on carefully designed synthetic data where we can control the degree of nonlinearity. We then follow up with a variety of applications coming from the nanophotonic component design domain as well as breast cancer data and gene expression data. Our results demonstrate that there exists a synergy between linear and nonlinear methods in that the proposed approach improves upon both baselines in most low data regimes while being comparable to the best performing method in all other situations.

\section{PCA-Boosted Autoencoders}

A central piece of our method is the use of the Parametric ReLu (PReLu, figure \ref{fig:prelu}) activation function for the Autoencoder \cite{he2015delving}. While the original motivation for this activation function was to improve the training process in large networks, we adopt if for different reasons. PReLU allows for the smooth transition from linear to nonlinear function approximation \cite{maas2013rectifier}. When we manually set the slope for the negative values of ReLU to 1, then the network is perfectly linear and in principle cannot do better than PCA for dimensionality reduction. In fact, emulating PCA is the best this network can do since PCA gives the optimal linear dimensionality reduction under the mean squared error loss function. The amount by which the slope for the negative values is adjusted becomes another parameter to be estimated from data  \cite{he2015delving}.


We assume that the architecture of the AE takes the vase-shaped form shown in Figure \ref{fig:AE}, where first the data is expanded in its dimensionality and then reduced back to the original dimension $n$ before being reduced to the bottleneck dimension $q$.  Although not necessary, it makes the robust weights initialization procedure easier to implement and present, while still allowing a significant degree of flexibility in the choice of the architecture. 



For the $n$-dimensional signal, let $\mb X$ be the original data matrix with $n$ columns (features) and $m$ rows (samples). We assume that $\mb X$ is already centered and possibly scaled, depending on the application needs. Let $ \hat{\mb X}$ be the result of representing $\mb X$ using just the first $q$ linear principal components. That is, 
$$ \hat{\mb X} = \mb U^{m \times q} \cdot \mb S^{q \times q} \cdot {\mb V^{n \times q}}^\top, $$
being a rank$-q$ singular value decomposition of matrix $\mb X$.

The functionality of a linear Autoencoder as in Figure \ref{fig:AE} with the bottleneck dimension $q$ can be described by the following equation:
\begin{align*}
    \mb x^\top \cdot \mb W_{e_1}^{n \times \cdot} &\cdots \mb W_{e_i}^{\cdot \times n} \cdot \mb W_{enc}^{n \times q} \cdot \mb W_{dec}^{q \times n} \cdot \mb W_{d_1}^{n \times \cdot} \\
    &\cdots \mb W_{d_i}^{\cdot \times n} = \hat{\mb x}^\top
\end{align*}
where $\mb W$-s are the weights of the AE. We seek to generate the weights such that the following equality holds:
\begin{align}
    \mb X \cdot \mb W_{e_1}^{n \times \cdot} &\cdots \mb W_{e_i}^{\cdot \times n} \cdot \mb W_{enc}^{n \times q} \cdot \mb W_{dec}^{q \times n} \cdot \mb W_{d_1}^{n \times \cdot} \nonumber \\
    &\cdots \mb W_{d_i}^{\cdot \times n} = \hat{\mb X}, \label{eq:PCAtoAE}
\end{align}
implying that the AE acts as PCA. 

The above equation implies that there is no unique solution to the initialization of weights due to both larger number of degrees of freedom (parameters) as well as scaling. This creates different possibilities for the initialization. We consider two options, both allowing a degree of randomness. The simplified version initializes all the layers randomly except for the bottleneck \footnote{Any random distribution is acceptable as long as it generates full rank matrices with probability 1.}, then computes the bottleneck layers (last encoder layer and first decoder layer) to match PCA. Let $\mb W_{enc-} = \prod_j \mb W_{e_j}$ and $\mb W_{dec+} = \prod_j \mb W_{d_j}$ be the weights corresponding to the product of all the encoder layers' weights except for the bottleneck, and the product of all the decoder layers' weights except for the bottleneck respectively. While $\mb W_{enc-}$ and $\mb W_{dec+}$ are products of random matrices, the bottleneck layer weights are calculated as follows:
$$ \mb W_{enc} = \mb W_{enc-}^\dagger \cdot \mb V ; ~~ \mb W_{dec} = {\mb V}^\top \cdot \mb W_{dec+}^\dagger,$$
where $\dagger$ is the Moore--Penrose pseudo-inverse of a matrix. It is easy to see that Eq. \eqref{eq:PCAtoAE} holds as a result. We call this procedure PCA-Naive.

The second, more involved initialization option maintains the numerical stability of the data throughout the network, and is outlined in Alg. \ref{alg:pca_robust} with its main step shown visually in Fig \ref{fig:robustinit}. The key difference is the way the weights for the non-bottleneck layers are initialized. Each such layer acts as a linear operator that is neither contractive nor expansive with respect to the original input space.

\begin{figure}[!htb]
  \centering
  
  \begin{minipage}{0.48\textwidth}
    \includegraphics[width=\columnwidth]{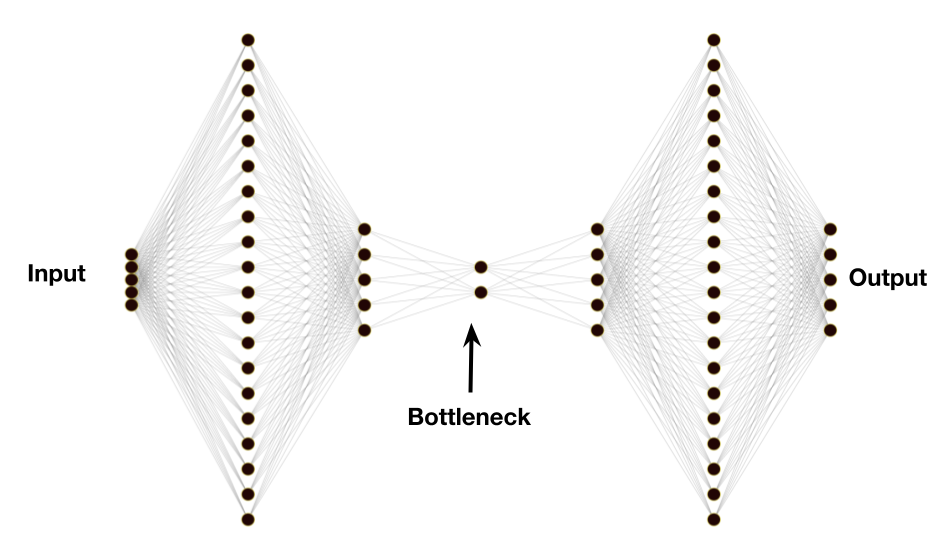}
    \caption{Autoencoder architecture used in experiments}
    \label{fig:AE}
    \end{minipage}
  \begin{minipage}{0.36\textwidth}
    \includegraphics[width=\columnwidth]{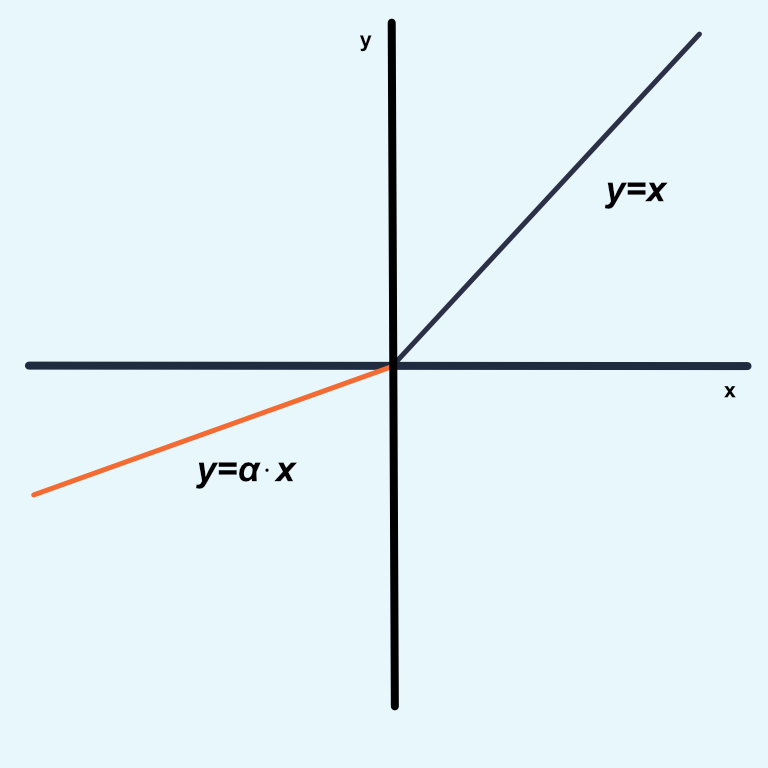}
    \caption{Behaviour of PReLu, during training, $\alpha$ is initially equal to 1 but is allowed to vary independently for each node}
    \label{fig:prelu}
    \end{minipage}
\end{figure}

\begin{figure}
\centering
    \includegraphics[width=0.8\linewidth]{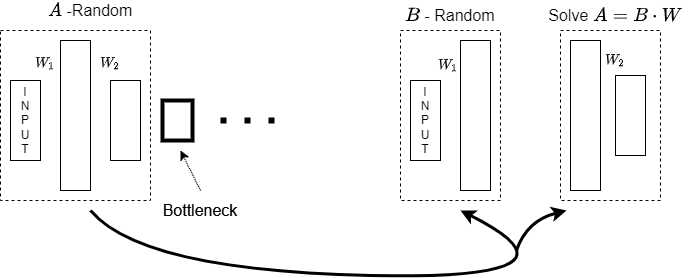}
    \caption{Example schematic of the stable initialization procedure of the encoder part of the AE described in Alg. \ref{alg:pca_robust}. $\mb A = \mb W_1 \cdot \mb W_2$, $\mb B = \mb W_1$ and $\mb W = \mb W_2$. $\mb A$ and $\mb B$ are random orthogonal matrices and $\mb W$ is a solution to Eq. \eqref{eq:robustInit}.\label{fig:robustinit}}
\end{figure}

\begin{lemma}
Let the Autoencoder weights be initialized as described in Alg. \ref{alg:pca_robust}. Then: 
\begin{enumerate}
    \item Correctness of the non-bottleneck layer initialization: line 15 in Alg. \ref{alg:pca_robust} is an exact solution to the following equation with an unknown $\mb W$:
\begin{equation}
    \mb A = \mb B \cdot \mb W \label{eq:robustInit}
\end{equation}
    \item Stability: for any input $\mb x \in \md R^n$ to the Autoencoder, the output at each layer before the bottleneck layer is norm preserving. Similarly, the output of the layers coming after the bottleneck layer is norm preserving. 
\end{enumerate} 
\end{lemma}
\begin{proof}
1. Let $\mb B \in \md R^{n \times p}$, with $p \ge n$ following the assumption about the AE architecture. Together with the fact that $\mb B$ is orthonormal we get $\mb B \mb B^\dagger = \mb I^{n \times n}$. Hence,
$ \mb B \cdot \mb W = \mb B \cdot \mb B^\dagger \mb A = \mb A$.\\
2. Note that the output at an encoder layer $e_k$ is $\mb x^\top \left(\prod_{j=1}^k \mb W_{e_j}\right)$. By construction in the Alg. \ref{alg:pca_robust}, $\mb A \triangleq \prod_{j=1}^k \mb W_{e_j} \in \md R^{n \times \cdot}$ corresponds to a random orthonormal matrix generated at a recursive iteration $(i-k)$. Since $n$ is the smaller dimension of the matrix, $||\mb x^\top \mb A|| = ||\mb x||$. An identical argument holds for the decoder layers coming after the bottleneck layer since it is generated the same way.
\end{proof}

\begin{algorithm}
\caption{Stable PCA-based weights initialization}\label{alg:pca_robust}
\textbf{Parameters:}\\
$\mb X \in \md R^{m \times n}$ - centered (optionally scaled) data matrix \\
$k$ - bottleneck dimension\\
$arch$ - list of dimensions for encoder/decoder layers\\
\textbf{Output:} List of weights\\
\begin{algorithmic}[1]
\State $\mb A = $\Call{RO}{$n,n$}
\State $\mb {Enc} = $\Call{RWI}{$\mb A,arch$}
\State $\mb W_{enc-} = \prod \mb {Enc}$ 
\State $\mb {Dec} = $\Call{RWI}{$\mb A,arch$}
\State $\mb W_{dec+} = \prod \mb {Dec}$ 
\State $\mb P = \Call{PCA}{\mb X,k}$
\State \Return List($\mb {Enc},\mb W_{enc-}^\dagger \mb P,\mb P^\dagger \mb W_{dec+}^\dagger,\mb {Dec}$)
\State
\State \#\# Robust weight initialization
\Function{RWI}{$\mb A \in \md R^{m\times n}, arch:$ List} 
\If{size of ($arch$) $\le$ 2}
\State \Return $\mb A$
\EndIf
\State $arch.pop()$ \Comment drop last element
\State $\mb B = $ \Call{RO}{$arch[0],arch[-1]$} \Comment -1 $\rightarrow$ last index
\State $\mb W = \mb B^\dagger \cdot \mb A$
\State \Return List(\Call{RWI}{$\mb B$,$arch$},$\mb W$)
\EndFunction
\State
\State \#\# Random orthonormal matrix
\Function{RO}{$m ,n$}
\State \Return random orthonormal matrix $ \in \md R^{m \times n}$
\EndFunction
\end{algorithmic}
\end{algorithm} 
The computational complexity of Alg. \ref{alg:pca_robust} is dominated by the time it takes to solve the largest linear system \eqref{eq:robustInit} or generation of random orthogonal matrices, depending on the approach. The upper bound on both is $O(l^3)$ where $l$ is the number of hidden units in the largest layer of AE, if singular value decomposition is used. This might be an obstacle for extremely large architectures, though such circumstances are unlikely due to the inherent scarcity of training data in the domains of interest we consider. In our use cases the time it takes to initialize the AE is negligible compared to the training time.

\section{Experiments}
The design of the experiments is intended to replicate a low-data regime with the downstream tasks requiring the choice of a single dimensionality reduction model that is later explored. The exploration of the latent space is however outside the scope of this work. The total available data in our experiments is split into $80\%$ for training, $10\%$ for validation, and $10\%$ for model selection, where the validation set is used to identify when to stop the training process. The model selection set is used to choose the best performing AE model among several randomly initialized and trained AE models. The same model selection is also applied to the AE models that reproduce PCA since this process allows significant degree of randomness as well. All final models are evaluated on a much larger testing set for analysis purposes.




Training is performed using the Adam optimizer with the slope of PReLU activation function for all the layers initialized to 1. Along with the PCA, the compared AE models are the PCA-Robust initialization (Alg. \ref{alg:pca_robust}), PCA-Naive initialization and Random initialization. Although it was tracked during the experiments, PCA-Naive is not depicted in the charts because it performed poorly---consistently worse and significantly less stable than the other methods. The Euclidean distance between the input and output vectors is used as the loss function, consistent with an implicit objective of PCA.

All autoencoder architectures we experiment with have 7 layers including the input and output layers. The choice of this architecture is somewhat arbitrary and was not optimized for any of the problems. 



Both synthetic and real experiments are repeated with different sample sizes (20, 30, 40, 80 and 100) representing varying sizes of the low data regimes. 

\subsection{Power Function}
In our first experiment the data is generated from a function that takes two parameters, satisfying:

\[
  x^n + y^n = z \label{eq:power} \tag{1}
\]

As shown in figure \ref{fig:ae_plots} , adjusting $n$ allows for the manipulation of the degree of curvature of the surface while maintaining the same range of a function for a fixed domain $x,y \in [0,1]$. As $n$ approaches $1$ the surface approximates a linear surface more closely.

In this experiment we study the effect of curvature and sample size on the respective performances of PCA, a Randomly-initialized AE and two AE's: one initialized with our stable algorithm (PCA-Robust), and one initialized naively without consideration for the stability issues discussed in the previous section (PCA-Naive). We vary the sample sizes (using 20, 30, 40, 50, 80 and 100) and different values for $n$ in equation \eqref{eq:power} ($n$= 1.1 ,nearly linear \ref{fig:ae_plots}(left), and, 4  clearly curved \ref{fig:ae_plots}(right)). Altogether, these values cover a range of curvature and data availability of the latent 2D subspace.

1000 points are randomly generated for each value of $n$ and from these 250 are set aside for the test set. The rest are used to draw samples from randomly. For each  data sample, the AEs were each initialized accordingly. All three autoencoders (PCA-Robust, PCA-Naive and Random) undergo the same training process on the given data sample. Finally, the experiment is repeated 200 times (resulting in different data samples) and the statistics are presented.

For the synthetic experiments, the layer sizes in sequence from the input layer to the output layer are 3-20-3-2-3-20-3, based on 3-dimensional data.

  
    
  

\begin{figure}[!tbp]
  \centering
  \begin{minipage}[b]{0.48\textwidth}
    \includegraphics[width=\columnwidth]{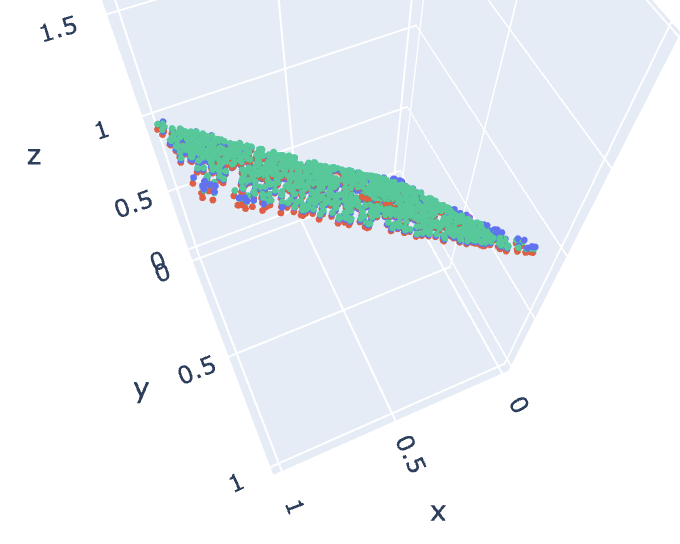}
    \label{fig:nearly-linear}
  \end{minipage}
  \hfill
  \begin{minipage}[b]{0.48\textwidth}
    \includegraphics[width=\columnwidth]{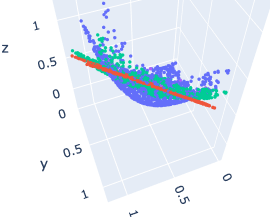}
  \end{minipage} 
  
  \caption{Randomly generated points (blue) satisfying \(x^n + y^n = z\) when $n=1.1$(left) and $n=4$(right) along with PCA projections(red) and PCA-Robust(green) for 40 training samples}
 \label{fig:ae_plots}
\end{figure}

Figure \ref{fig:results} shows the average errors of different models evaluated on the test data (size=250) as a function of the training data size. The results show that PCA-Robust performs significantly better than PCA for all cases. The result for Random for $n=1.1$ are not shown as they are significantly worse.

Moreover, when curvature is increased, similar performance patterns hold and all AE-based methods begin to do better than PCA for larger data sizes as expected. Most importantly, PCA-Robust AE consistently outperforms other models in most data regimes and at least as well as any other model otherwise.

\begin{figure}[h]
  \begin{tabular}{c c}
    \includegraphics[width=0.5\textwidth]{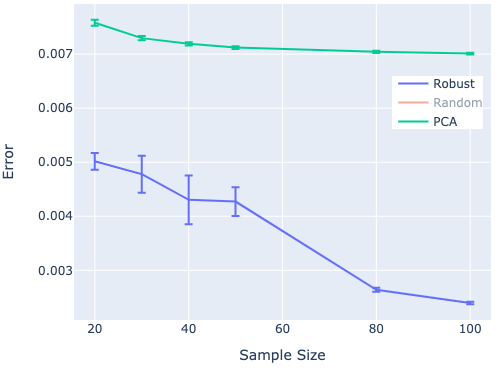} &
    \includegraphics[width=0.5\textwidth]{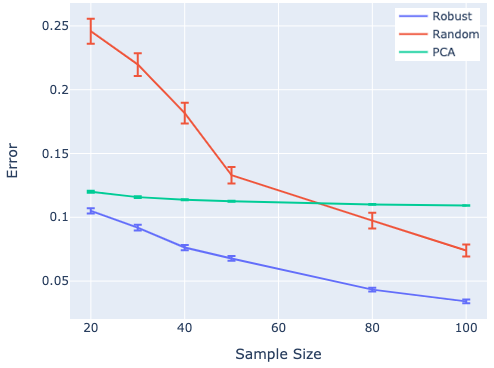}
  \end{tabular}
  \caption{{Comparison of dimensionality reduction techniques on a synthetic dataset generated with curvature 1.1 (Left) and 4 (Right). \label{fig:test_synth}}}
  \label{fig:results}
\end{figure}







\subsection{Nanophotonic Component Design}

Here we present several real-world problems that motivated this work in the first place, and illustrate the practical utility of dimensionality reduction with limited data. In nanophotonic component design, the geometry  on a chip are manipulated on a nanometer scale to generate functions for applications ranging from optical communication to biological sensing \cite{cheben2018subwavelength}. Designing such devices requires solving Maxwell’s partial differential equations to obtain the electromagnetic field distribution---a computationally expensive process. Such design problems are typically posed as optimization problems where parameters represent physical quantities such as material properties and/or geometry. The introduction of an adjoint based optimization method in nanophotonic design \cite{lalau2013adjoint,jensen2011topology} allowed an efficient gradient-based optimization that easily scales up to a large number of parameters. Acquiring a single optimized design typically takes hours to days of computation depending on the problem and computational resources.  However, in practice a single optimized design is often not the best one to be fabricated. The preferred course of action is to present a collection of optimized designs among which the designer chooses one or a handful to manufacture, based on a variety of considerations, such as other figures of merit or fabrication reliability, that are not captured in the objective function. 

In \cite{melati2019mapping,dezfouli2020perfectly,melati2020design}, PCA was used to characterize a subspace of optimized designs from a small collection of such designs. Exploring this lower dimensional subspace became computationally feasible by simple sampling. While the linear approach was useful, the data also exhibited some curvature \cite{melati2019mapping}. Using a method that can capture the lower dimensional subspace more accurately, including its curvature, implies that exploring this subspace will identify a collection of yet better performing designs to consider for manufacturing. Here we consider three design problems - two grating couplers and a power splitter, as described below - where we demonstrate that the proposed autoencoders can capture some of the curvature of the design space and generalize well despite having limited data. \\

\begin{figure}[h]
  \begin{tabular}{c c}
    \includegraphics[width=0.4\textwidth]{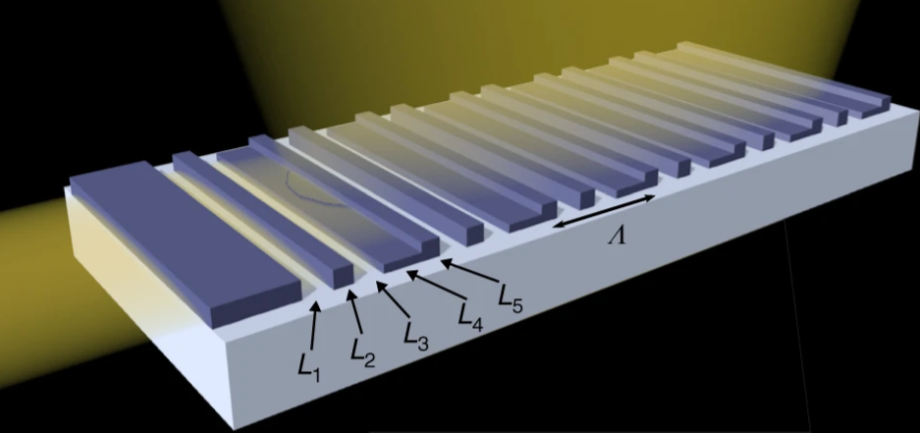} &
    \includegraphics[width=0.56\textwidth]{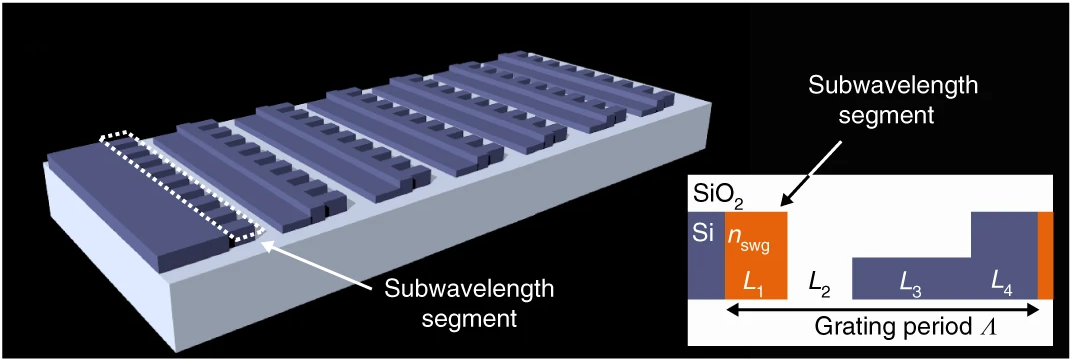}
  \end{tabular}
  \caption{Schematic representations of grating couplers' structures: Left - Grating Coupler 1; Right - Grating coupler 2. (reproduced with permission from \cite{melati2019mapping}). \label{fig:couplers}}
\end{figure}

\noindent \textbf{Grating Coupler 1}\\


A vertical grating coupler is a device that diffracts light injected in-plane vertically upwards, or coupling light signals in the photonic chip to and from optical fibers, or to free space input and output beams. It is a necessary device for connecting photonic chips to the surrounding environment. A schematic of the device that is considered here is shown in Figure \ref{fig:couplers} (Left).

The complete dataset of optimized structures for this vertical grating coupler consists of 540 designs, also referred to as a set of good designs. These were obtained from computationally intensive simulation-based optimizations and selected from a set consisting of more than 30,000 candidate designs based on an optical performance criterion (a waveguide-fiber coupling efficiency of at least 74\%). The designs are characterized by five segment values ($L_1$ to $L_5$) representing different widths of silicon blocks, which are the parameters to the design problem (see Figure  \ref{fig:couplers} (Left)). In \cite{melati2019mapping}, it was shown that two principal components were enough to capture most of the optimized design subspace within this five dimensional design problem. Only a small set of such designs were required to identify the subspace. 
The layer sizes for this experiment were 5-20-5-2-5-20-5 (see Figure \ref{fig:AE}) for the 5-dimensional data. Note that the size, $n$, of the bottleneck for the experiments is compared to PCA with $n$ principal components as in \cite{melati2019mapping}.


\begin{figure}[ht]
  \begin{tabular}{l r}
    \includegraphics[width=0.5\textwidth]{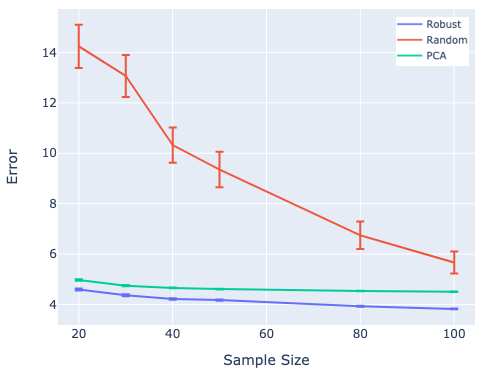} &
    \includegraphics[width=0.5\textwidth]{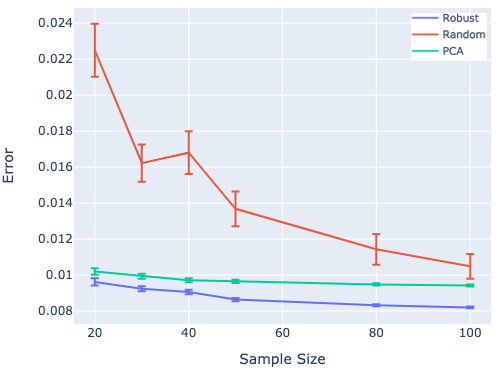}
  \end{tabular}
  \caption{Comparison of dimensionality reduction techniques on Grating Coupler 1 (Left) and Grating Coupler 2 (Right) datasets. Robust(Blue) depicts the average projection error of PCA-Robust AE, Random (Red) is of Randomly initialized AE, and PCA (Green) is using PCA. \label{fig:couplers_performance}}
\end{figure}

Similar to the setup in \cite{melati2019mapping}, we choose 2 as a reduced dimension and simulate scenarios where only a small number of these optimized designs are available, ranging from 20 to 100. These are randomly sampled from the available good designs set. To corroborate our entire experimental setup that can be sensitive to data partitioning due to the low data regime, we set aside half of the available designs designs (270) as the test set. The experiments are further repeated 50 times creating different data splits to measure accurate statistics of the results. The error bars represent the uncertainty of an average estimate.

Figure \ref{fig:couplers_performance} (Left) shows the performance of different methods evaluated on the test set. We notice that our proposed Robust model consistently outperforms randomly-initialized autoencoders as well as PCA across all dataset sizes. As expected, the performance of randomly-initialized autoencoders slowly catches up and is expected to match the performance of Robust model when sufficient data is provided.\\

\noindent \textbf{Grating Coupler 2}\\


A similar experiment was conducted on data obtained from another vertical grating coupler design, as depicted in Figure \ref{fig:couplers} (Right) Compared to the previous example, here a section based on a subwavelength metamaterial is introduced in the device structure, as described in \cite{melati2019mapping,dezfouli2020perfectly}. While the number of variables defining the structure is also five - $(L_1,\cdots,L_4,n_{swg})$, the variable $n_{swg}$ represents a different physical quantity - the effective material index of the subwavelength metamaterial section - which has a different interpretation as well as a different order of magnitude\cite{cheben2018subwavelength}. The other four variables, as previously, represent silicon segment widths.

The complete dataset consists of 1502 optimized designs (coupling efficiency >$74\%$) of which half were used as the test set. Since the variables in this design problem are of different magnitudes, the values were scaled for all dimensionality reduction methods. The structure of the experiment is otherwise identical to the previous grating coupler design setup.

The performance results are shown in Figure \ref{fig:couplers_performance} (Right). Trends very similar to those  of the Grating Coupler 1 experiment are observed.\\

\noindent \textbf{Power splitter}\\

Another nanoponotonic component ubiquitous in integrated circuits is a power splitter. Its role is to split the incoming wave carrying the signal into two or more ports propagating the same signal but with reduced intensity. We consider a parameterized design of a 1x2 power splitter as in \cite{melati2020design} targeting equal splitting ratio on the output ports . The splitter is defined by a single silicon shape, the parameters directly controlling its boundary. Examples of such optimized splitters are shown in Figure \ref{fig:splitters}. Ten parameters define the boundary at equally spaced location along the X axis, while the boundary points in between them are interpolated and smoothed.

The Y-Splitter data is 10-dimensional and we compare different dimensionality reductions, specifically, 4 and 5. This resulted in layer sizes of 10-20-10-4-10-20-10 and 10-20-10-5-10-20-10 respectively.

The complete dataset consists of 645 designs that achieve at least 97\% coupling efficiency, defined as the amount of power coupled into the output ports (since the device is symmetric the power is guaranteed to be split equally). The optimized structures are obtained from hundreds of gradient based optimizations with random initial conditions. Figures \ref{fig:ysplit4} and \ref{fig:ysplit5} represent the performance results on this dataset for two different choices of the reduced dimension. Similar to previous experiments,  PCA-Robust method outperforms its competitors except for PCA at the smallest dataset regime, where their performance is comparable.

An interesting outcome of the reduction showed in these two experiments is that we observe that with 100 samples, the average projection error for PCA when reduced to 5 dimensions ($108\pm 0.7$) is marginally better than that of PCA-Robust reduced to 4 ($114\pm1.4$). First, it suggests that PCA needs more dimensions to capture approximately the same amount of information as a nonlinear method -- a sign of an increasingly curved subspace. Second, exploring a lower-dimensional subspace will often require (significantly) less data samples. It is particularly important in applications such as this one, where sampling is costly.

\begin{figure}[h]
\begin{center}
  \begin{tabular}{c c}
    \includegraphics[width=0.42\textwidth]{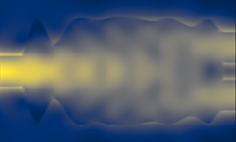} \hspace{10pt}&
    \includegraphics[width=0.415\textwidth]{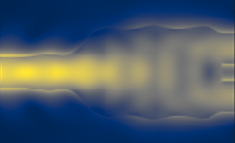}
  \end{tabular}
  \caption{Examples of two optimized 1x2 power splitters. The light travels from the left port (waveguide) and being split equally into two output ports (waveguides) on the right. Both power splitters are defined by their (symmetric) boundaries at the top and bottom representing the silicon shape. Yellow color represents the electric field intensity profile. \label{fig:splitters}}
\end{center}
\end{figure}

\begin{figure}[!htb]
  \centering
  \begin{minipage}{0.48\textwidth}
    \includegraphics[width=\columnwidth]{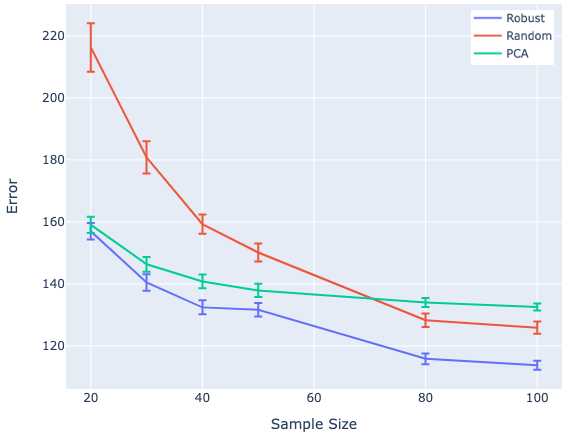}
    \caption{Performance for Y-splitter reduced to 4 dimensions}
    \label{fig:ysplit4}
  \end{minipage}
  \hfill
  \begin{minipage}{0.48\textwidth}
    \includegraphics[width=\columnwidth]{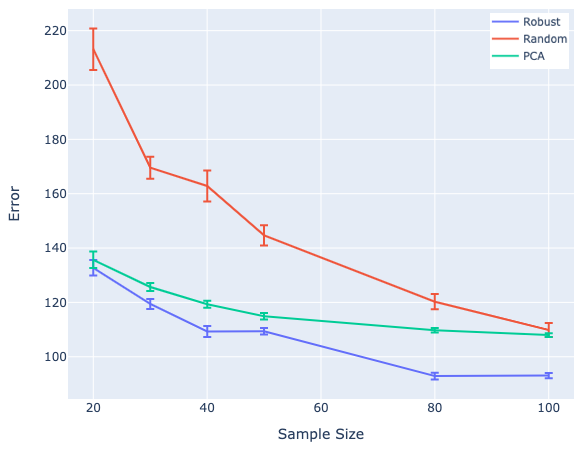}
    \caption{Performance for Y-splitter reduced to 5 dimensions}
    \label{fig:ysplit5}
  \end{minipage}
\end{figure}

\section{Breast Cancer Wisconsin Dataset}
We now describe the application of our method to a dataset from a different domain, one in which the downstream task is different. Whereas in the nanophotonics case, the purpose of performing dimensionality reduction was to enable efficient mapping of the optimized design subspace, in the breast cancer dataset dimensionality reduction is performed in order to improve the performance of a classifier. 

The Breast Cancer Wisconsin Dataset \cite{Dua:2019}, included in sci-kit learn package\cite{scikit-learn},  consists of 569 total samples labeled ``benign'' and ``malignant''. These labels are irrevelant for our purposes since we're interested in comparing the average projection error in the dimensionality reduction of PCA and PCA-Robust. Each sample contains 30 real-valued features describing charecteristics of the cell nuclei depicted in an image of a fine-needle aspirate of a breast mass. Since the feature size is 30 we had to expand the width of our vase used in our earlier experiment. Our resulting vase had a ``width'' of 100 nodes and a bottleneck of 15 nodes (to result in 15 reduced dimensions). The result layers were 30-100-30-15-30-100-30.

Figure \ref{fig:b_cancer} shows that, the PCA-Robust autoencoder outperforms other methods on all except for the smallest data set size, where its performance is comparable to that of PCA. 

\begin{figure}[!htb]
  \centering
  \end{figure}

\begin{figure}[!htb]
  \centering
  \begin{minipage}{0.48\textwidth}
    \includegraphics[width=\columnwidth]{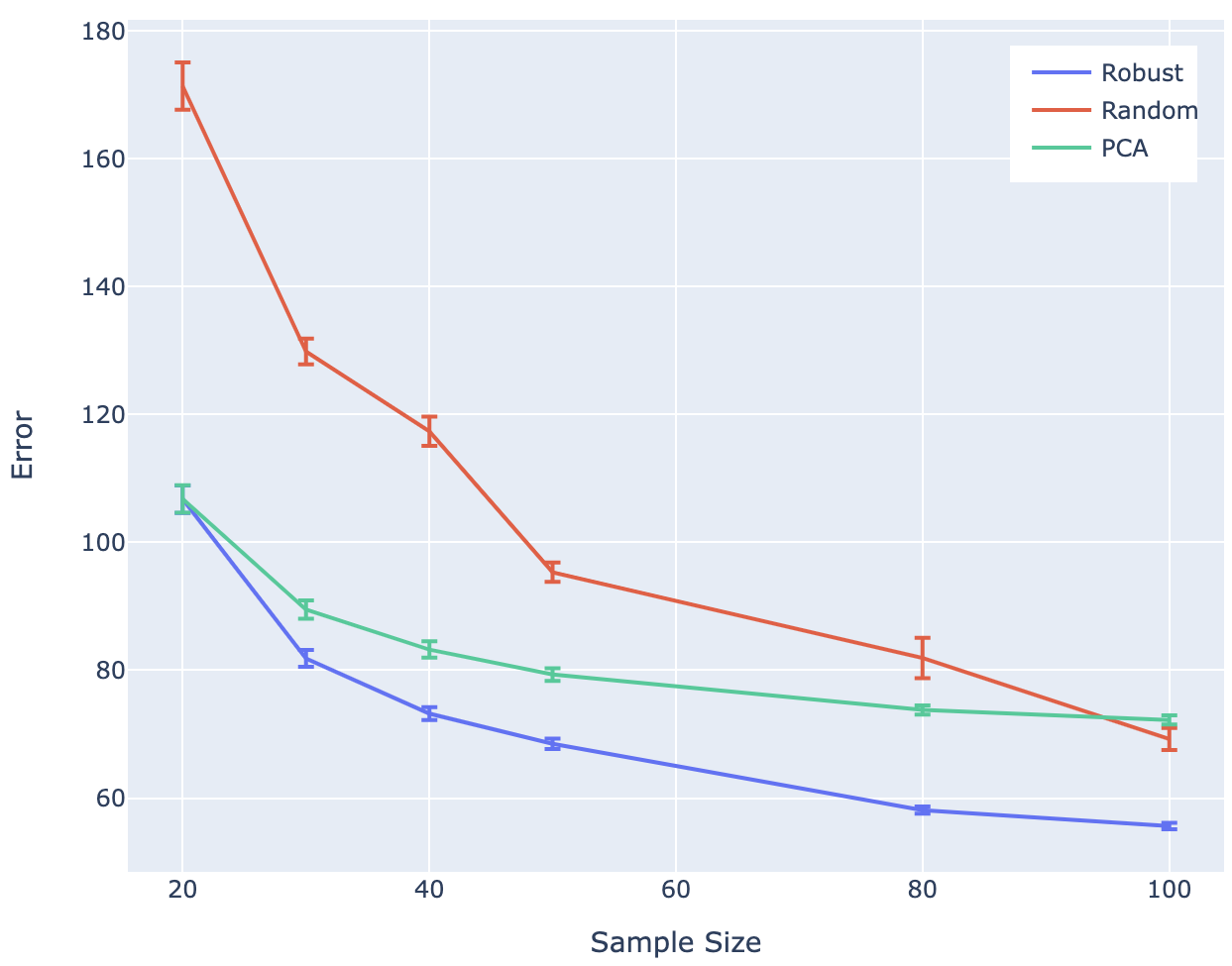}
    \caption{Performance for Breast Cancer Data reduced to 15 dimensions}
    \label{fig:b_cancer}
  \end{minipage}
  \hfill
  \begin{minipage}{0.48\textwidth}
    \includegraphics[width=\columnwidth]{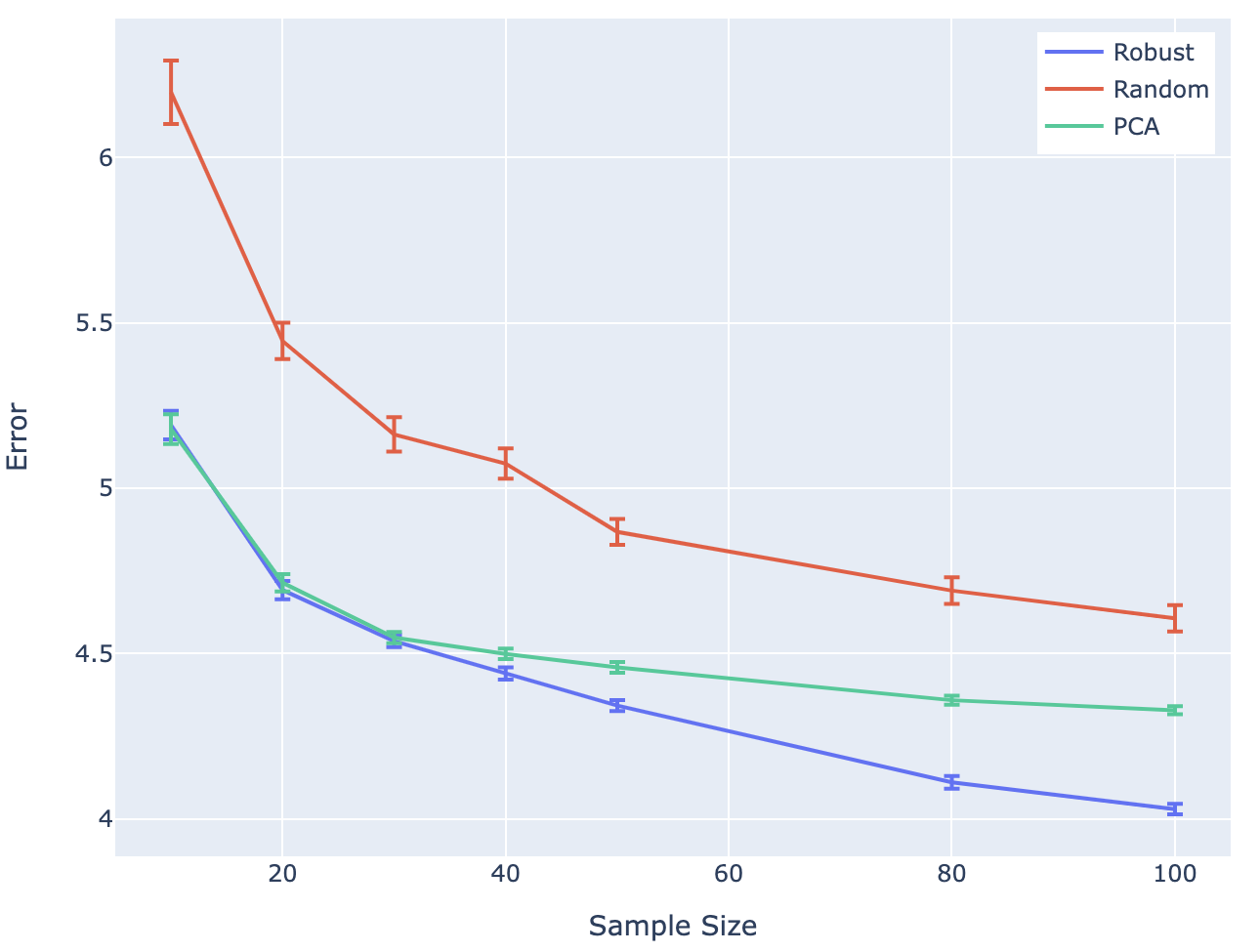}
    \caption{Performance for the cotton dataset}
    \label{fig:cotton}
  \end{minipage}
\end{figure}

\section{Gene Expression}
Another application that often offers only limited data where PCA is  used is gene expression. We consider a subset of gene expression data associated with varying conditions of fungal stress in cotton  \cite{bedre2015genome}. The data contains 662 samples. Each sample has six features, each is real-valued and indicates a condition associated with fungal stress in cotton. PCA is applied to this data to enable the visualization of the relationships between samples and thus a reduction that preserves more information(lower average projection error) can proivide a clearer representation of these relationships. 

In our experiments we reduce the samples to two dimensions. The autoencoder layers were 6-20-6-2-6-20-6.

Figure \ref{fig:cotton} shows the comparison of the three methods. We can observe that PCA-Robust autoencoder maintains the same performance as PCA for a number of data sizes, suggesting that the amount of data is not yet sufficient to pick up any nonlinearity that generalizes well. For larger data sizes, there is sufficient data to improve upon the linear method and PCA-Robust approach leverages that information.

\section{Conclusion and Future Work}
In this work we demonstrated that autoencoders, with proper initialization, can offer a viable solution for dimensionality reduction even in a regime of limited data. Our results show that there is no need to choose between linear or nonlinear model fitting on small datasets. Instead, the stable PCA initialization provides the best of both worlds, allowing the AE training to introduce nonlinearity only to the extent that the data allows. These results are encouraging in domains where only PCA has been used to reduce the dimensionality due to very limited availability of datasets. Furthermore, due to the explicit encoding-decoding nature of the autoencoders, sampling the low-dimensional space for the purpose of applications such as nanophotonic component design is straightforward.

While the idea of initializing the autoencoders from PCA is not new, we demonstrate that, unlike its naive implementation, a numerically stable initialization is critical to the training and the performance of the final model. We conjecture that the stable initialization we propose turns out to be significantly more favourable for the typical gradient-based training method(s) of AE. More detailed investigation of this matter is left for future work.

Finally, initialization of the network weights in the proposed fashion might offer a degree of stability and robustness similar to the proposals that address adversarial example issues in neural networks and enforce general smoothness \cite{cisse2017parseval,anil2019sorting}. This would be another exciting avenue to explore.



%
%
%
\bibliographystyle{splncs04}
\bibliography{references}

\end{document}